\documentclass[global]{svjour}  

\usepackage{algorithm}
\usepackage{algorithmic}
\usepackage{psfig}
\usepackage{url}
\usepackage{haldefs}

\journalname{Machine Learning}

\newcommand{\newcite}[1]{\cite{#1}}

\newcommand{\searn}{\textsc{Searn}}
\newcommand{\bayesum}{\textsc{BayeSum}}
\newcommand{\mmmn}{M$^3$N}
\newcommand{\mmmns}{M$^3$Ns}
\newcommand{\SVMISO}{SVM$^\textnormal{struct}$}
\newcommand{\SVMISOs}{SVM$^\textnormal{struct}$s}

\newcommand{\lab}[1]{$_{\textsf{#1}}$}
\newcommand{\jointlab}[2]{$_{\textsf{#1}}^{\textsf{#2}}$}

\newcommand{\w}{\vec w}

\newcommand{\slow}{$\sim$}
\newcommand{\unav}{$-$}

\newcommand{\cmax}{c_{\textrm{max}}}
\newcommand{\new}{{}^{\textrm{new}}}

\newcommand{\lCS}{\ell^{\textrm{CS}}_h(h')}

\begin{document}

\title{Search-based Structured Prediction}
\author{Hal Daum\'e III\inst{1} \and John Langford\inst{2} \and Daniel Marcu\inst{3}}
\institute{%
  School of Computing, University of Utah, Salt Lake City, UT 84112%
\and%
  Yahoo! Research Labs, New York, NY 10011%
\and%
  Information Sciences Institute, Marina del Rey, CA 90292%
}

\date{Received: ??? / Revised version: ???}

\maketitle

\begin{abstract}
We present \searn, an algorithm for integrating \textsc{sear}ch and
l\textsc{earn}ing to solve complex structured prediction problems such
as those that occur in natural language, speech, computational
biology, and vision.  \searn\ is a meta-algorithm that transforms
these complex problems into simple classification problems to which
any binary classifier may be applied.  Unlike current algorithms for
structured learning that require \emph{decomposition} of both the loss
function and the feature functions over the predicted structure,
\searn\ is able to learn prediction functions for \emph{any} loss
function and \emph{any} class of features.  Moreover, \searn\ comes
with a strong, natural theoretical guarantee: good performance on the
derived classification problems implies good performance on the
structured prediction problem.
\end{abstract}

\section{Introduction} \label{sec:intro}

Prediction is the task of learning a function $f$ that maps inputs $x$
in an input domain $\cX$ to outputs $y$ in an output domain $\cY$.
Standard algorithms---support vector machines, decision trees, neural
networks, etc.---focus on ``simple'' output domains such as $\cY = \{
-1, +1 \}$ (in the case of binary classification) or $\cY = \R$ (in
the case of univariate regression).

We are interested in problems for which elements $y \in \cY$ have
complex internal structure.  The simplest and best studied such output
domain is that of labeled sequences.  However, we are interested in
even more complex domains, such as the space of English sentences (for
instance in a machine translation application), the space of short
documents (perhaps in an automatic document summarization
application), or the space of possible assignments of elements in a
database (in an information extraction/data mining application).  The
structured complexity of features and loss functions in these problems
significantly exceeds that of sequence labeling problems.

From a high level, there are four dimensions along which structured
prediction algorithms vary: structure (varieties of structure for
which efficient learning is possible), loss (different loss functions
for which learning is possible), features (generality of feature
functions for which learning is possible) and data (ability of
algorithm to cope with imperfect data sources such as missing data,
etc.).  An in-depth discussion of alternative structured prediction
algorithms is given in Section~\ref{sec:alternatives}.  However, to
give a flavor, the popular conditional random field algorithm
\cite{lafferty01crf} is viewed along these dimensions as follows.
Structure: inference for a CRF is tractable for any graphical model
with bounded tree width; Loss: the CRF typically optimizes a log-loss
approximation to 0/1 loss over the entire structure; Features: any
feature of the input is possible but only output features that obey
the graphical model structure are allowed; Data: EM can cope with
hidden variables.

We prefer a structured prediction algorithm that is not limited to
models with bounded treewidth, is applicable to any loss function, can
handle arbitrary features and can cope with imperfect data.  Somewhat
surprisingly, \searn\ meets nearly all of these requirements by
transforming structured prediction problems into binary prediction
problems to which a vanilla binary classifier can be applied.  \searn\
comes with a strong theoretical guarantee: good binary classification
performance implies good structured prediction performance.  Simple
applications of \searn\ to standard structured prediction problems
yield tractable state-of-the-art performance.  Moreover, we can apply
\searn\ to more complex, non-standard structured prediction problems
and achieve excellent empirical performance.

This paper has the following outline:
\begin{enumerate}
\item Introduction.
\item Core Definitions.
\item The \searn\ Algorithm.
\item Theoretical Analysis.
\item A Comparison to Alternative Techniques.  
\item Experimental results.
\item Discussion.
\end{enumerate}

\section{Core Definitions}

In order to proceed, it is useful to formally define a structured
prediction problem in terms of a state space.

\begin{definition} \label{def:sp}
A \emph{structured prediction} problem $\cD$ is a cost-sensitive
classification problem where $\cY$ has structure: elements $y \in \cY$
decompose into variable-length vectors
$(y_1,y_2,\dots,y_T)$.\footnote{Treating $y$ as a vector is simply a
useful encoding; we are not interested only in sequence labeling
problems.}  $\cD$ is a distribution over inputs $x \in \cX$ and cost
vectors $\vec c$, where $\card{\vec c}$ is a variable in $2^T$.
\end{definition}

As a simple example, consider a parsing problem under F$_1$ loss.  In
this case, $\cD$ is a distribution over $(x,\vec c)$ where $x$ is an
input sequence and for all trees $y$ with $\card{x}$-many leaves,
$c_y$ is the F$_1$ loss of $y$ on the ``true'' output.

The goal of structured prediction is to find a function $h : \cX \fto
\cY$ that minimizes the loss given in Eq~\eqref{eq:sp-loss}.

\begin{equation} \label{eq:sp-loss}
L(\cD, h) = \Ep_{(x,\vec c) \sim \cD} \left\{ c_{h(x)} \right\}
\end{equation}

The algorithm we present is based on the view that a vector $y \in
\cY$ can be produced by predicting each component $(y_1, \dots, y_T)$
in turn, allowing for dependent predictions.  This is important for
coping with general loss functions.  For a data set $(x_1,c_1), \dots,
(x_N,c_N)$ of structured prediction examples, we write $T_n$ for the
length of the longest search path on example $n$, and $T_\textrm{max}
= \max_n T_n$.

\section{The \searn\ Algorithm} \label{sec:searn}

There are several vital ingredients in any application of \searn: a
seach space for decomposing the prediction problem; a cost sensitive
learning algorithm; labeled structured prediction training data; a
known loss function for the structured prediction problem; and a good
initial policy.  These aspects are described in more detail below.

\begin{description}
\item [A search space $\cS$.]  The choice of search space plays a role
similar to the choice of structured decomposition in other algorithms.
Final elements of the search space can always be referenced by a
sequence of choices $\hat{\vec{y}}$.  In simple applications of
\searn\ the search space is concrete.  For example, it might consist
of the parts of speech of each individual word in a sentence.  In
general, the search space can be abstract, and we show this can be
beneficial experimentally.  An abstract search space comes with an
(unlearned) function $f(\hat{\vec{y}})$ which turns any sequence of
predictions in the abstract search space into an output of the correct
form.  (For a concrete search space, $f$ is just the identity
function.  To minimize confusion, we will leave off $f$ in future
notation unless its presence is specifically important.)

\item [A cost sensitive learning algorithm $A$.]  The learning
algorithm returns a multiclass classifier $h(s)$ given cost sensitive
training data.  Here $s$ is a description of the location in the
search space.  A reduction of cost sensitive classification to binary
classification \cite{beygelzimer05reductions} reduces the requirement
to a binary learning algorithm.  \searn\ relies upon this learning
algorithm to form good generalizations.  Nothing else in the \searn\
algorithm attempts to achieve generalization or estimation.  The
performance of \searn\ is strongly dependent upon how capable the
learned classifier is.  We call the learned classifier a {\em policy}
because it is used multiple times on inputs which it effects, just as
in reinforcement learning.

\item [Labeled structured prediction training data.]  \searn\ digests
the labeled training data for the structured prediciton problem into
cost-sensitive training data which is fed into the cost-sensitive
learning algorithm.\footnote{A $k$-class cost-sensitive example is
given by an input $X$ and a vector of costs $\vec c \in (\R^+)^k$.
Each class $i$ has an associated cost $c_i$ and the goal is a function
$h : X \mapsto i$ that minimizes the expected value of $c_i$.  See
\cite{beygelzimer05reductions}.}

\item [A known loss function.]  A loss function
  $L(\vec{y},f(\hat{\vec{y}}))$ must be known and be computable for
  any sequence of predictions.

\item [A good initial policy.]  This policy should achieve low loss
  when applied to the training data.  This can (but need not always)
  be defined using a search algorithm.
\end{description}

\subsection{\searn\ at Test Time}

\searn\ at test time is a very simple algorithm.  It uses the policy
returned by the learning algorithm to construct a sequence of
decisions $\hat{\vec{y}}$ and makes a final prediction
$f(\hat{\vec{y}})$.  First, one uses the learned policy to compute
$y_0$ on the basis of just the input $x$.  One then computes $y_1$ on
the basis of $x$ and $y_0$, followed by predicting $y_2$ on the basis
of $x$, $y_0$ and $y_1$, etc.  Finally, one predicts $y_T$ on the
basis of the input $x$ and \emph{all} previous decisions.

\subsection{\searn\ at Train Time} \label{sec:searn:training}

\searn\ operates in an iterative fashion.  At each iteration it uses a
known policy to create new cost-sensitive classification examples.
These examples are essentially the classification decisions that a
policy would need to get right in order to perform search well.  These
are used to learn a new classifier, which is interpreted as a new
policy.  This new policy is interpolated with the old policy and the
process repeats.

\subsubsection{Initial Policy}  \label{sec:searn:training:op}

\searn\ relies on a good initial policy on the \emph{training data}.
This policy can take full advantage of the training data labels.  The
initial policy needs to be efficiently computable for \searn\ to be
efficient.  The implications of this assumption are discussed in
detail in Section~\ref{sec:searn:policy:op}, but it is strictly weaker
than assumptions made by other structured prediction techniques.  The
initial policy we use is a policy that, for a given state predicts the
best action to take with respect to the labels:

\begin{definition}[Initial Policy] \label{def:op}
For an input $x$ and a cost vector $\vec c$ as in Def~\ref{def:sp},
and a state $s = x \times (y_1, \dots, y_t)$ in the search space, the
\emph{initial policy} $\pi(s,\vec c)$ is $\arg\min_{y_{t+1}}
\min_{y_{t+2}, \dots, y_T} c_{\langle y_1, \dots, y_T \rangle}$.  That
is, $\pi$ chooses the action (i.e., value for $y_{t+1}$) that
minimizes the corresponding cost, assuming that all \emph{future}
decisions are also made optimally.
\end{definition}

This choice of initial policy is optimal when the correct output is a
deterministic function of the input features (effectively in a
noise-free environment).

\subsubsection{Cost-sensitive Examples}  \label{sec:searn:training:cost}

In the training phase, \searn\ uses a given policy $h$ (initialized to
the the initial policy $\pi$) to construct cost-sensitive multiclass
classification examples from which a new classifier is learned.  These
classification examples are created by \emph{running} the given policy
$h$ over the training data.  This generates one path per structured
training example.  \searn\ creates a single cost-sensitive example for
each state on each path.  The classes associated with each example are
the available actions (or next states).  The only difficulty lies in
specifying the costs.

The cost associated with taking an action that leads to state $s$ 
is the \emph{regret} associated with this action, given our current
policy.  For each state $s$ and each action $a$, we take action $a$
and then execute the policy to gain a full sequence of predictions
$\hat{\vec{y}}$ for which we can compute a loss
$\vec{c}_{\hat{\vec{y}}}$.  Of all the possible actions, one, $a'$,
has the minimum expected loss.  The \emph{cost}
$\ell_h(\vec{c},s,a)$ for an action $a$ in state $s$ is the
difference in loss between taking action $a$ and taking the action
$a'$; see Eq~\eqref{eq:loss}.

\begin{equation} \label{eq:loss}
\ell_h (\vec{c},s,a) = \Ep_{\hat{\vec{y}} \sim (s,a,h)}
{c}_{\hat{\vec{y}}} - \min_{a'} \Ep
_{\hat{\vec{y}} \sim (s,a',h)} {c}_{\hat{\vec{y}}}
\end{equation}

One complication arises because the policy used may be stochastic.
This can occur even when the base classifier learned is deterministic
due to stochastic interpolation within \searn.  There are (at least) three
possible ways to deal with randomness.

\begin{enumerate}
\item Monte-Carlo sampling: one draws many paths according to $h$
  beginning at $s'$ and average over the costs.

\item Single Monte-Carlo sampling: draw a single path and use the
  corresponding cost, with tied randomization as per Pegasus
  \cite{ng00pegasus}.

\item Approximation: it is often possible to efficiently compute the
  loss associated with following the initial policy from a given
  state; when $h$ is sufficiently good, this may serve as a useful and
  fast approximation.  (This is also the approach described by
  \newcite{langford05reinforcement}.)
\end{enumerate}

The quality of the learned solution depends on the quality of the
approximation of the loss.  Obtaining Monte-Carlo samples is likely
the best solution, but in many cases the approximation is sufficient.
An empirical comparison of these options is performed in
\cite{daume06thesis}.  Here it is observed that for easy problems (one
for which low loss is possible), the approximation performs
approximately as well as the alternatives.  Moreover, typically the
approximately outperforms the single sample approach, likely due to
the noise induced by following a single sample.

\subsubsection{Algorithm}  \label{sec:searn:training:algorithm}

The \searn\ algorithm is shown in Figure~\ref{fig:algorithm}.  As
input, the algorithm takes a structured learning data set, an initial
policy $\pi$ and a multiclass cost sensitive learner $A$.  \searn\
operates iteratively, maintaining a current policy hypothesis $h$ at
each iteration.  This hypothesis is initialized to the initial policy
(step 1).

\begin{figure}[t]
\center\framebox{
\hspace{-2mm}\begin{minipage}[t]{11.5cm}
~{\bf Algorithm} \searn($S^{\textrm{SP}}$, $\pi$, $A$)
\begin{algorithmic}[1]
\STATE Initialize policy $h \leftarrow \pi$
\WHILE{ $h$ has a significant dependence on $\pi$ }
\STATE Initialize the set of cost-sensitive examples $S \leftarrow \emptyset$
\FOR{$(x,y) \in S^{\textrm{SP}}$}
\STATE Compute predictions under the current policy $\hat{\vec{y}} \sim x,h$ 
\FOR{$t = 1 \dots T_x$}
\STATE Compute features $\Phi = \Phi(s_t)$ for state $s_t = (x,y_1,...,y_t)$
\STATE Initialize a cost vector $\vec c = \langle \rangle$
\FOR{each possible action $a$}
\STATE Let the cost $\ell_a$ for example $x,\vec{c}$ at state $s$ be $\ell_h (\vec{c},s,a)$
\ENDFOR
\STATE Add cost-sensitive example $(\Phi,\vec{\ell})$ to $S$
\ENDFOR
\ENDFOR
\STATE Learn a classifier on $S$: $h' \leftarrow A(S)$
\STATE Interpolate: $h \leftarrow \be h' + (1-\be) h$
\ENDWHILE
\STATE {\bf return} $h_{\textrm{last}}$ without $\pi$
\end{algorithmic}
\end{minipage}
}
\caption{Complete \searn\ Algorithm}
\label{fig:algorithm}
\end{figure}

The algorithm then loops for a number of iterations.  In each
iteration, it creates a (multi-)set of cost-sensitive examples, $S$.
These are created by looping over each structured example (step 4).
For each example (step 5), the \emph{current policy} $h$ is used to
produce a full output, represented as a sequence of predictions
$y_1,...,y_{T_n}$.  From this, states are derived and used to create a
single cost-sensitive example (steps 6-14) at each timestep.

The first task in creating a cost-sensitive example is to compute the
associated feature vector, performed in step 7.  This feature vector
is based on the current state $s_t$ which includes the features $x$
(the creation of the feature vectors is discussed in more detail in
Section~\ref{sec:searn:training:features}).  The cost vector contains
one entry for every possible action $a$ that can be executed from
state $s_t$.  For each action $a$, we compute the expected loss
associated with the state $s_t \oplus a$: the state arrived at
assuming we take action $a$ (step 10).  

\searn\ creates a large set of cost-sensitive examples $S$.  These are
fed into any cost-sensitive classification algorithm, $A$, to produce
a new classifier $h'$ (step 15).  In step 16, \searn\ \emph{combines}
the newly learned classifier $h'$ with the current classifier $h$ to
produce a new classifier.  This combination is performed through
stochastic interpolation with interpolation parameter $\be$ (see
Section~\ref{sec:searn:analysis} for details).  The meaning of
stochastic interpolation here is: ``every time $h$ is evaluated, a new
random number is drawn.  If the random number is less than $\be$ then
$h'$ is used and otherwise the old $h$ is used.''  \searn\ returns the
final policy with $\pi$ removed (step 18) and the stochastic
interpolation renormalized.

\subsection{Feature Computations} \label{sec:searn:training:features}

In step 7 of the \searn\ algorithm (Figure~\ref{fig:algorithm}), one
is required to compute a feature vector $\Phi$ on the basis of the
give state $s_t$.  In theory, this step is arbitrary.  However, the
performance of the underlying classification algorithm (and hence the
induced structured prediction algorithm) hinges on a good choice for
these features.  The feature vector $\Phi(s_t)$ may depend on
\emph{any} aspect of the input $x$ and any past decision.  In
particular, there is no limitation to a ``Markov'' dependence on
previous decisions.

For concreteness, consider the part-of-speech tagging task: for each
word in a sentence, we must assign a single part of speech (eg.,
\textsf{Det}, \textsf{Noun}, \textsf{Verb}, etc.).  Given a state $s_t
= \langle x, y_1, \dots y_t \rangle$, one might compute a sparse
feature vector $\Phi(s_t)$ with zeros everywhere except at positions
corresponding to ``interesting'' aspects of the input.  For instance,
a feature corresponding to the identity of the $t+1$st word in the
sentence would likely be very important (since this is the word to be
tagged).  Furthermore, a feature corresponding to the value $y_t$
would likely be important, since we believe that subsequent tags are
not independent of previous tags.  These features would serve as the
input to the cost-sensitive learning algorithm, which would attempt to
predict the correct label for the $t+1$st word.  This usually
corresponds to learning a single weight vector for each class (in a
one-versus-all setting) or to learning a single weight vector for
each pair of classes (for all-pairs).

\subsection{Policies} \label{sec:searn:policy}

\searn\ functions in terms of policies, a notion borrowed from the
field of reinforcement learning.  This section discusses the nature of
the initial policy assumption and the connections to reinforcement
learning.

\subsubsection{Computability of the Initial Policy} \label{sec:searn:policy:op}

\searn\ relies upon the ability to start with a good initial policy
$\pi$, defined formally in Definition~\ref{def:op}.  For many simple
problems under standard loss functions, it is straightforward to
compute a good policy $\pi$ in constant time.  For instance, consider
the sequence labeling problem (discussed further in
Section~\ref{sec:sequence}).  A standard loss function used in this
task is Hamming loss: of all possible positions, how many does our
model predict incorrectly.  If one performs search left-to-right,
labeling one element at a time (i.e., each element of the $y$ vector
corresponds exactly to one label), then $\pi$ is trivial to compute.
Given the correct label sequence, $\pi$ simply chooses at position $i$
the correct label at position $i$.  However, \searn\ is not limited to
simple Hamming loss.  A more complex loss function often considered
for the sequence segmentation task is F-score over (correctly labeled)
segments.  As discussed in
Section~\ref{sec:sequence:problems:chunking}, it is just as easy to
compute a good initial policy for this loss function.  This is not
possible in many other frameworks, due to the non-additivity of
F-score.  This is independent of the features.

This result---that \searn\ can learn under strictly more complex
structures and loss functions than other techniques---is not limited
to sequence labeling, as demonstrated below in
Theorem~\ref{thm:powerful}.  In order to prove this, we need to
formalize what we consider as ``other techniques.''  We use the
max-margin Markov network (\mmmn) formalism \cite{taskar05mmmn} for
comparison, since this currently appears to be the most powerful
generic framework.  In particular, learning in \mmmns\ is often
tractable for problems that would be \#P-hard for conditional random
fields.  The \mmmn\ has several components, one of which is the
ability to compute a loss-augmented minimization \cite{taskar05mmmn}.
This requirement states that Eq~\eqref{eq:mmmn:lossaug} is computable
for any input $x$, output set $\cY_x$, true output $y$ and weight
vector $\vec w$.

\begin{equation} \label{eq:mmmn:lossaug}
\textit{opt}(\cY_x, y, \vec w) =
  \arg\max_{\hat y \in \cY_x} 
    \vec w \T \Phi(x,\hat y) -
    l(y,\hat y)
\end{equation}

In Eq~\eqref{eq:mmmn:lossaug}, $\Phi(\cdot)$ produces a vector of features,
$\vec w$ is a weight vector and $l(y,\hat y)$ is the loss for
prediction $\hat y$ when the correct output is $y$.

\begin{theorem} \label{thm:powerful}
Suppose Eq~\eqref{eq:mmmn:lossaug} is computable in time $T(x)$; then
the optimal policy is computable in time $\cO(T(x))$.  Further, there
exist problems for which the optimal policy is computable in constant
time and for which Eq~\eqref{eq:mmmn:lossaug} is an NP-hard computation.
\end{theorem}

\begin{proof}[sketch]
For the first part, we use a vector encoding of $y$ that maintains the
decomposition over the regions used by the \mmmn.  Given a prefix
$y_1, \dots, y_t$, solve $\textit{opt}$ on the future choices (i.e.,
remove the part of the structure corresponding to the first $t$
outputs), which gives us an optimal policy.

For the second part, we simply make $\Phi$ complex: for instance,
include long-range dependencies in sequence labeling.  As the Markov
order $k$ increases, the complexity of Viterbi decoding grows as
$l^k$, where $l$ is the number of labels.  In the limit as the Markov
order approaches the length of the longest sequence, $T_\textrm{max}$,
the computation for the minimal cost path (with or without the added
complexity of augmenting the cost with the loss) becomes NP-hard.
Despite this intractability for Viterbi decoding, \searn\ can be
applied to the identical problem with the exact same feature set, and
inference becomes tractable (precisely because \searn\ never applies a
Viterbi algorithm).  The complexity of one iteration of \searn\ for
this problem is identical to the case when a Markov assumption is
made: it is $\order{Tl^b}$, where $T$ is the length of the sequence,
and $b$ is the beam size.
\end{proof}

\subsubsection{Search-based Policies} \label{sec:searn:policy:search-based}

The \searn\ algorithm and the theory to be presented in
Section~\ref{sec:searn:analysis} do not require that the initial
policy be optimal.  \searn\ can train against \emph{any} policy.  One
artifact of this observation is that we can use \emph{search} to
create the initial policy.

At any step of \searn, we need to be able to compute the best next
action.  That is, given a node in the search space, and the cost
vector $\vec c$, we need to compute the best step to take.  This is
\emph{exactly} the standard search problem: given a node in a search
space, we find the shortest path to a goal.  By taking the first step
along this shortest path, we obtain a good initial policy (assuming
this shortest path is, indeed, shortest).  This means that when
\searn\ asks for the best next step, one can execute \emph{any}
standard search algorithm to compute this, for cases where a good
initial policy is not available analytically.

Given this observation, the requirements of \searn\ are reduced: instead
of requiring a good initial policy, we simply require that one can
perform efficient approximate search.  

\subsubsection{Beyond Greedy Search} \label{sec:searn:policy:beyond}

We have presented \searn\ as an algorithm that mimics the operations
of a \emph{greedy} search algorithm.  Real-world experience has shown
that often greedy search is insufficient and more complex search
algorithms are required.  This observation is consistent with the
standard view of search (trying to find a shortest path), but nebulous
when considered in the context of \searn.  Nevertheless, it is often
desirable to allow a model to trade past decisions off future
decisions, and this is precisely the purpose of instituting more
complex search algorithms.

It turns out that any (non-greedy) search algorithm operating in a
search space $\cS$ can be equivalently viewed as a greedy search
algorithm operating in an abstract space $\cS^*$ (where the structure
of the abstract space is dependent on the original search algorithm).
In a general search algorithm \cite{russell95aibook}, one maintains a
\emph{queue} of active states and expands a single state in each
search step.  After expansion, each resulting child state is
enqueued.  The ordering (and, perhaps, maximal size) of the queue is
determined by the specific search algorithm.

In order to simulate this more complex algorithm as greedy search, we
construct the abstract space $\cS^*$ as follows.  Each node $s \in
\cS^*$ represents a state of the queue.  A transition exists between
$s$ and $s'$ in $\cS^*$ exactly when a particular expansion of an
$\cS$-node in the $s$-queue results in the queue becoming $s'$.
Finally, for each goal state $g \in \cS$, we augment $\cS^*$ with a
single unique goal state $g^*$.  We insert transitions from $s \in
\cS^*$ to $g^*$ exactly when $g^* \in s$.  Thus, in order to complete
the search process, a goal node must be in the queue and the search
algorithm must select this single node.

In general, \searn\ makes no assumptions about how the search process
is structured.  A different search process leads to a different bias
in the learning algorithm.  It is up to the designer to construct a
search process so that (a) a good bias is exhibited and (b) computing
a good initial policy is easy.  For instance, for some combinatorial
problems such as matchings or tours, it is known that left-to-right
beam search tends to perform poorly.  For these problems, a local
hill-climbing search is likely to be more effective since we expect it
to render the underlying classification problem simpler.

\section{Theoretical Analysis} \label{sec:searn:analysis}

\searn\ functions by slowly moving \emph{away} from the initial policy
(which is available only for the training data) \emph{toward} a fully
learned policy.  Each iteration of \searn\ \emph{degrades} the current
policy.  The main theorem states that the learned policy is not much
worse than the starting (optimal) policy plus a term related to the
average cost sensitive loss of the learned classifiers and another
term related to the maximum cost sensitive loss.  To simplify
notation, we write $T$ for $T_\textrm{max}$.

It is important in the analysis to refer explicitly to the error of
the classifiers learned during \searn\ process.  Let $\searn(\cD,h)$
denote the distribution over classification problems generated by
running \searn\ with policy $h$ on distribution $\cD$.  Also let
$\lCS$ denote the loss of classifier $h'$ on the distribution
$\searn(\cD,h)$.  Let the average cost sensitive loss over $I$
iterations be:

\begin{equation} \label{eq:convergence}
\ell_{\textrm{avg}} =
  \frac 1 I
    \sum_{i=1}^I
      \ell^{\textrm{cs}}_{h_i}(h_i')
\end{equation}
where $h_i$ is the $i$th policy and $h_i'$ is the classifier learned
on the $i$th iteration.

\begin{theorem} \label{cor}
For all $\cD$ with $\cmax = \Ep_{(x,\vec c) \sim \cD} \max_{\vec y}
c_{\vec y}$ (with $(x,\vec c)$ as in Def~\ref{def:sp}), for all
learned cost sensitive classifiers $h'$, \searn\ with $\be = 1/T^3$
and $2T^3 \ln T$ iterations, outputs a learned policy with loss
bounded by:

\begin{equation*}
L(\cD,h_{\textrm{last}}) \leq 
L(\cD,\pi) 
+ 2 T \ell_{\textrm{avg}} \ln T 
+ (1 + \ln T) \cmax / T
\end{equation*}
\end{theorem}

The dependence on $T$ in the second term is due to the cost sensitive
loss being an average over $T$ timesteps while the total loss is a
sum.  The $\ln T$ factor is not essential and can be removed using
other approaches \cite{PSDP} \cite{langford05reinforcement}.  The
advantage of the theorem here is that it applies to an algorithm that
naturally copes with variable length $T$ and yields a smaller amount
of computation in practice.

The choices of $\beta$ and the number of iterations are pessimistic in
practice.  Empirically, we use a development set to perform a line
search minimization to find per-iteration values for $\beta$ and to
decide when to stop iterating.  The analytical choice of $\beta$ is
made to ensure that the probability that the newly created policy only
makes \emph{one} different choice from the previous policy for any
given example is sufficiently low.  The choice of $\beta$ assumes the
worst: the newly learned classifier \emph{always} disagrees with the
previous policy.  In practice, this rarely happens.  After the first
iteration, the learned policy is typically quite good and only rarely
differs from the initial policy.  So choosing such a small value for
$\beta$ is unneccesary: even with a higher value, the current
classifier often agrees with the previous policy.

The proof rests on the following lemmae.

\begin{lemma}[Policy Degradation] \label{lemma:degradation}
Given a policy $h$ with loss $L(\cD,h)$, apply a single iteration of
\searn\ to learn a classifier $h'$ with cost-sensitive loss $\lCS$.
Create a new policy $h\new$ by interpolation with parameter $\beta \in
(0,1/T)$.  Then, for all $\cD$, with $\cmax = \Ep_{(x,\vec c) \sim
\cD} \max_i c_i$ (with $(x,\vec c)$ as in Def~\ref{def:sp}):

\begin{equation} \label{eq:degradation}
L(\cD,h\new) \leq L(\cD,h) + T \be \lCS + \frac 1 2 \be^2 T^2 \cmax
\end{equation}
\end{lemma}

\begin{proof}
The proof largely follows the proofs of Lem 6.1 and Theorem 4.1 for
conservative policy iteration \cite{kakade02cpi}.  The three
differences are that (1) we must deal with the finite horizon case;
(2) we move \emph{away from} rather than \emph{toward} a good policy;
and (3) we expand to higher order.

The proof works by separating three cases depending on whether
$h^\textrm{CS}$ or $h$ is called in the process of running $h\new$.
The easiest case is when $h^\textrm{CS}$ is never called.  The second
case is when it is called exactly once.  The final case is when it is
called more than once.  Denote these three events by $c=0$, $c=1$ and
$c\geq2$, respectively.

\begin{align}
L(\cD,h\new)
  = &  Pr(c=0)    L(\cD, h\new \| c=0 ) \nonumber\\
    &+ Pr(c=1)    L(\cD, h\new \| c=1 ) \nonumber\\
    &+ Pr(c\geq2) L(\cD, h\new \| c\geq 2 ) \\
&\nonumber\\
\leq&  (1-\be)^T L(\cD, h) 
    + T\be(1-\be)^{T-1} \Big[ L(\cD,h) + \lCS \Big] \\
    &+ \Big[1 - (1-\be)^T - T\be(1-\be)^{T-1}\Big] \cmax \nonumber\\
&\nonumber\\
  = & L(\cD, h) + T\be(1-\be)^{T-1} \lCS \\
  &+ \Big[ 1- (1-\be)^T -T \be ( 1 - \be) ^{T-1} \Big](\cmax - L(\cD,h))
&\nonumber\\
&\nonumber\\
\leq&  L(\cD, h) + T\be \lCS \nonumber\\
    & + \Big[1 - (1-\be)^T - T\be(1-\be)^{T-1}\Big] \cmax \\
&\nonumber\\
   =&  L(\cD, h) + T\be \lCS 
     + \left( \sum_{i=2}^T (-1)^i \be^i {T \choose i} \right) \cmax \\
&\nonumber\\
\leq&  L(\cD, h) + T\be\lCS   
     + \frac 1 2 T^2 \be^2 \cmax
\end{align}

The first inequality writes out the precise probability of the events
in terms of $\be$ and bounds the loss of the last event ($c > 2$) by
$\cmax$.  The second inequality is algebraic.  The third uses
the assumption that $\be < 1/T$.  
\end{proof}

\noindent
This lemma states that applying a single iteration of \searn\ does not
cause the structured prediction loss of the learned hypothesis to
degrade too much.  In particular, up to a first order approximation,
the loss increases proportional to the loss of the learned classifier.
This observation can be iterated to yield the following lemma:

\begin{lemma}[Iteration] \label{thm:convergence}
For all $\cD$, for all learned $h'$, after $C/\be$ iterations of
\searn\ beginning with a policy $\pi$ with loss $L(\cD,\pi)$, and
average learned losses as Eq~\eqref{eq:convergence}, the loss of the
final learned policy $h$ (without the optimal policy component) is
bounded by Eq~\eqref{eq:bound}.

\begin{equation} \label{eq:bound}
L(\cD,h_{\textrm{last}}) \leq 
L(\cD,\pi) + C T \ell_{\textrm{avg}} + \cmax \left(\frac 1 2 C T^2 \be + T\exp[-C]\right)
\end{equation}
\end{lemma}

This lemma states that after $C/\be$ iterations of \searn\ the learned
policy is not much worse than the quality of the initial policy $\pi$.
The theorem follows from a choice of the constants $\be$ and $C$ in
Lemma~\ref{thm:convergence}.

\begin{proof}
The proof involves invoking Lemma~\ref{lemma:degradation} $C/\be$
times.  The second and the third terms sum to give the following:

\begin{equation*}
L(\cD,h) \leq L(\cD,\pi) + C T \ell_{\textrm{avg}} + \cmax \left(\frac 1 2 C T^2 \be\right)
\end{equation*}

Last, if we call the initial policy, we fail with loss at most
$\cmax$.  The probability of failure after $C/\be$ iterations is at
most $T(1-\be)^{C/\be} \leq T \exp[-C]$.
\end{proof}

\section{Comparison to Alternative Techniques} \label{sec:alternatives}

Standard techniques for structured prediction focus on the case where
the $\arg\max$ in Eq~\eqref{eq:argmax} \emph{is} tractable.  Given its
tractability, they attempt to learn parameters $\th$ such that solving
Eq~\eqref{eq:argmax} often results in low loss.  There are a handful
of classes of such algorithms and a large number of variants of each.
Here, we focus on \emph{independent classifier} models,
\emph{perceptron-based} models, and \emph{global} models (such as
conditional random fields and max-margin Markov networks).  There are,
of course, alternative frameworks (see, eg.,
\cite{weston02kde,mcallester04cfd,altun04gp,mcdonald04margin,tsochantaridis05svmiso}),
but these are common examples.

\subsection{The $\arg\max$ Problem} \label{sec:argmax}

Many structured prediction problems construct a scoring function $F(y
\| x, \th)$.  For a given input $x \in \cX$ and set of parameters $\th
\in \Th$, $F$ provides a score for each possible output $y$.  This
leads to the ``$\arg\max$'' problem (also known as the decoding
problem or the pre-image problem), which seeks to find the $y$ that
maximizes $F$ in order to make a prediction.

\begin{equation} \label{eq:argmax}
\hat y = \arg\max_{y \in \cY_x} F(y \| x, \th)
\end{equation}

In Eq~\eqref{eq:argmax}, we seek the output $y$ from the
set $\cY_x$ (where $\cY_x \subseteq \cY$ is the set of all
``reasonable'' outputs for the input $x$ -- typically assumed to be
finite).  Unfortunately, solving Eq~\eqref{eq:argmax} exactly is
tractable only for very particular structures $\cY$ and scoring
functions $F$.  As an easy example, when $\cY_x$ is interpreted as a
label sequence and the score function $F$ depends only on adjacent
labels, then dynamic programming can be used, leading to an
$\order(nk^2)$ prediction algorithm, where $n$ is the length of the
sequence and $k$ is the number of possible labels for each element in
the sequence.  Similarly, if $\cY$ represents trees and $F$ obeys a
context-free assumption, then this problem can be solved in time
$\order(n^3k)$.

Often we are interested in more complex structures, more complex
features or both.  For such tasks, an exact solution to
Eq~\eqref{eq:argmax} is not tractable.  For example, In natural
language processing most statistical word-based and phrase-based
models of translation are known to be NP-hard
\cite{germann03decoding}; syntactic translations models based on
synchronous context free grammars are sometimes polynomial, but with
an exponent that is too large in practice, such as $n^{11}$
\cite{huang05hooks}.  Even in comparatively simple problems like
sequence labeling and parsing---which are only $\order(n)$ or
$\order(n^3)$---it is often still computationally prohibitive to
perform exhaustive search \cite{bikel03collins}.  For another sort of
example, in computational biology, most models for phylogeny
\cite{foulds82phylogeny} and protein secondary structure prediction
\cite{crescenzi98folding} result in NP-hard search problems.

When faced with such intractable search problem, the standard tactic
is to use an approximate search algorithm, such as greedy search, beam
search, local hill-climbing search, simulated annealing, etc.  These
search algorithms are unlikely to be provably optimal (since this
would imply that one is efficiently solving an NP-hard problem), but
the hope is that they perform well on problems that are observed in
the real world, as opposed to ``worst case'' inputs.

Unfortunately, applying suboptimal search algorithms to solve the
structured prediction problem from Eq~\eqref{eq:argmax} dispenses with
many nice theoretical properties enjoyed by sophisticated learning
algorithms.  For instance, it may be possible to learn Bayes-optimal
parameters $\th$ such that \emph{if} exact search were possible, one
would always find the best output.  But given that exact search is not
possible, such properties go away.  Moreover, given that different
search algorithms exhibit different properties and biases, it is easy
to believe that the value of $\th$ that is optimal for one search
algorithm is not the same as the value that is optimal for another
search algorithm.\footnote{In fact, \newcite{wainwright06wrong} has
provided evidence that when using approximation algorithms for
graphical models, it is important to use the same approximate at both
training and testing time.}  It is these observations that have
motivated our exploration of \emph{search-based} structured prediction
algorithms: learning algorithms for structured prediction that
explicitly model the search process.

\subsection{Independent Classifiers}

There are essentially two varieties of local classification techniques
applied to structured prediction problems.  In the first variety, the
structure in the problem is ignored, and a single classifier is
trained to predict each element in the output vector independently
\newcite{punyakanok01inference} or with dependence created by
enforcement of membership in $\cal{Y}_x$ constraints
\newcite{punyakanok05constrainted}.  The second variety is typified by
maximum entropy Markov models \cite{mccallum00memm}, though the basic
idea of MEMMs has also been applied more generally to SVMs
\cite{kudo01chunking,kudo03kernel,gimenez04svmtook}.  In this variety,
the elements in the prediction vector are made sequentially, with the
$n$th element conditional on outputs $n-k \dots n-1$ for a $k$th order
model.

In the purely independent classifier setting, both training and
testing proceed in the obvious way.  Since the classifiers make one
decision completely independently of any other decision, training
makes use only of the input.  This makes training the classifiers
incredibly straightforward, and also makes prediction easy.  In fact,
running \searn\ with $\Phi(x,y)$ independent of all but $y_n$ for the
$n$ prediction would yield exactly this framework (note that there
would be no reason to iterate \searn\ in this case).  While this
renders the independent classifiers approach attractive, it is also
significantly weaker, in the sense that one cannot define complex
features over the output space.  This has not thus far hindered its
applicability to problems like sequence labeling
\cite{punyakanok01inference}, parsing and semantic role labeling
\cite{punyakanok05role}, but does seem to be an overly strict
condition.  This also limits the approach to Hamming loss.

\searn\ is more similar to the MEMM-esque prediction setting.  The key
difference is that in the MEMM, the $n$th prediction is being made on
the basis of the $k$ previous predictions.  However, these predictions
are noisy, which potentially leads to the suboptimal performance
described in the previous section.  The essential problem is that the
models have been trained assuming that they make all previous
predictions correctly, but when applied in practice, they only have
predictions about previous labels.  It turns out that this can cause
them to perform nearly arbitrarily badly.  This is formalized in the
following theorem, due to Matti K{\"a}{\"a}ri{\"a}inen.

\begin{theorem}[\newcite{kaariainen06lower}] \label{thm:kaar}
There exists a distribution $\cD$ over first order binary Markov
problems such that training a binary classifier based on true previous
predictions to an error rate of $\ep>0$ leads to a Hamming loss given
in Eq~\eqref{eq:kaar}, where $T$ is the length of the sequence.

\begin{equation} \label{eq:kaar}
\frac T 2
 - \frac {1 - (1 - 2 \ep)^{T+1}} {4 \ep}
 + \frac 1 2
\approx
\frac T 2
\end{equation}

\noindent
Where the approximation is true for small $\ep$ or large $T$.
\end{theorem}

Recently, \cite{cohen05stacked} has described an algorithm termed
\emph{stacked sequential learning} that attempts to remove this bias
from MEMMs in a similar fashion to \searn.  The stacked algorithm
learns a sequence of MEMMs, with the model trained on the $t+1$st
iteration based on outputs of the model from the $t$th iteration.  For
sequence labeling problems, this is quite similar to the behaviour of
\searn\ when $\be$ is set to $1$.  However, unlike \searn, the stacked
sequential learning framework is effectively limited to sequence
labeling problems.  This limitation arises from the fact that it
implicitly assumes that the set of decisions one must make in the
future are always going to be same, regardless of decisions in the
past.  In many applications, such as entity detection and tracking
\cite{daume05coref}, this is not true.  The set of possible choices
(actions) available at time step $i$ is heavily dependent on past
choices.  This makes the stacked sequential learning inapplicable in
these problems.

\subsection{Perceptron-Style Algorithms}

The structured perceptron is an extension of the standard perceptron
\cite{rosenblatt58perceptron} to structured prediction
\cite{collins02perceptron}.  Assuming that the $\arg\max$ problem is
tractable, the structured perceptron constructs the weight vector in
nearly an identical manner as for the binary case.  While looping
through the training data, whenever the predicted $\hat y_n$ for $x_n$
differs from $y_n$, we update the weights according to
Eq~\eqref{eq:perceptron-update}.

\begin{equation} \label{eq:perceptron-update}
\w \leftarrow \w + \Phi(x_n,y_n) - \Phi(x_n,\hat y_n)
\end{equation}

This weight update serves to bring the vector closer to the true
output and further from the incorrect output.  As in the standard
perceptron, this often leads to a learned model that generalizes
poorly.  As before, one solution to this problem is weight averaging
\cite{freund99perceptron}.

The \emph{incremental perceptron} \cite{collins04incremental} is a
variant on the structured perceptron that deals with the issue that
the $\arg\max$ may not be analytically available.  The idea of the
incremental perceptron is to replace the $\arg\max$ with a beam search
algorithm.  The key observation is that it is often possible to detect
in the process of executing search whether it is possible for the
resulting output to ever be correct.  The incremental perceptron is
essentially a search-based structured prediction technique, although
it was initially motivated only as a method for speeding up
convergence of the structured perceptron.  In comparison to \searn, it
is, however, much more limited.  It cannot cope with arbitrary loss
functions, and is limited to a beam-search application.  Moreover, for
search problems with a large number of internal decisions (such as
entity detection and tracking \cite{daume05coref}), aborting search at
the first error is far from optimal.

\subsection{Global Prediction Algorithms} \label{sec:ml:sp:crf}

Global prediction algorithms attempt to learn parameters that,
essentially, rank correct (low loss) outputs higher than incorrect
(high loss) alternatives.

Conditional random fields are an extension of logistic regression
(maximum entropy models) to structured outputs \cite{lafferty01crf}.
Similar to the structured perceptron, a conditional random field does
not employ a loss function, but rather optimizes a log-loss
approximation to the 0/1 loss over the entire output.  Only when the
features and structure are chosen properly can dynamic programming
techniques be used to compute the required partition function, which
typically limits the application of CRFs to linear chain models under
a Markov assumption.

The maximum margin Markov network (\mmmn) formalism considers the
structured prediction problem as a quadratic programming problem
\cite{taskar03mmmn,taskar05mmmn}, following the formalism for the
support vector machine for binary classification.  The \mmmn\
formalism extends this to structured outputs under a given loss
function $l$ by requiring that the \emph{difference} in score between
the true output $y$ and any incorrect output $\hat y$ is at least the
loss $l(x,y,\hat y)$ (modulo slack variables).  That is: the \mmmn\
framework scales the \emph{margin} to be proportional to the loss.
Under restrictions on the output space and the features (essentially,
linear chain models with Markov features) it is possible to solve the
corresponding quadratic program in polynomial time.

In comparison to CRFs and \mmmns, \searn\ is strictly more general.
\searn\ is limited neither to linear chains nor to Markov style
features and can effectively and efficiently optimize structured
prediction models under far weaker assumptions (see
Section~\ref{sec:summarization} for empirical evidence supporting this
claim).

\section{Experimental Results}

In this section, we present experimental results on two different
sorts of structured prediction problems.  The first set of
problems---the sequence labeling problems---are comparatively simple
and are included to demonstrate the application of \searn\ to easy
tasks.  They are also the most common application domain on which
other structured prediction techniques are tested; this enables us to
directly compare \searn\ with alternative algorithms on standardized
data sets.  The second application we describe is based on an
automatic document summarization task, which is a significantly more
complex domain than sequence labeling.  This task enables us to test
\searn\ on significantly more complex problems with loss functions
that do not decompose over the structure.

\subsection{Sequence Labeling} \label{sec:sequence}

Sequence labeling is the task of assigning a \emph{label} to each
element in an input sequence.  Sequence labeling is an attractive test
bed for structured prediction algorithms because it is the simplest
non-trivial structure.  Modern state-of-the-art structured prediction
techniques fare very well on sequence labeling problems.  In this
section, we present a range of results investigating the performance
of \searn\ on four separate sequence labeling tasks: handwriting
recognition, named entity recognition (in Spanish), syntactic chunking
and joint chunking and part-of-speech tagging.

For pure sequence labeling tasks (i.e., when segmentation is not also
done), the standard loss function is Hamming loss, which gives credit
on a per label basis.  For a true output $y$ of length $N$ and
hypothesized output $\hat y$ (also of length $N$), Hamming loss is
defined according to Eq~\eqref{eq:sequence:loss:hamming}.

\begin{align}
\ell^\textrm{Ham}(y,\hat y) &\defeq \sum_{n=1}^N \Ind\big[ y_n \neq \hat y_n \big]
  \label{eq:sequence:loss:hamming}
\end{align}

The most common loss function for joint segmentation and labeling
problems (like the named entity recognition and syntactic chunking
problems) is F$_1$ measure over chunks\footnote{We note in passing
that directly optimizing F$_1$ may not be the best approach, from the
perspective of integrating information in a pipeline
\cite{manning06f1}.  However, since F$_1$ is commonly used and does
\emph{not} decompose over the output sequence, we use it for the
purposes of demonstration.}.  F$_1$ is the geometric mean of precision
and recall over the (properly-labeled) chunk identification task,
given in Eq~\eqref{eq:sequence:loss:f}.

\begin{equation}
\ell^\textrm{F}(y,\hat y) \defeq
  \frac {2 \card{y \cap \hat y}} {\card{y} + \card{\hat y}}
  \label{eq:sequence:loss:f}
\end{equation}

As can be seen in Eq~\eqref{eq:sequence:loss:f}, one is penalized both
for identifying too many chunks (penalty in the denominator) and for
identifying too few (penalty in the numerator).  The advantage of
F$_1$ measure over Hamming loss seen most easily in problems where the
majority of words are ``not chunks''---for instance, in gene name
identification \cite{mcdonald05gene}---Hamming loss often prefers
a system that identifies \emph{no} chunks to one that identifies some
correctly and other incorrectly.  Using a weighted Hamming loss can
not completely alleviate this problem, for essentially the same
reasons that a weighted zero-one loss cannot optimize F$_1$ measure in
binary classification, though one can often achieve an approximation
\cite{lewis01filtering,musicant03fmeasure}.

\subsubsection{Handwriting Recognition} \label{sec:sequence:problems:handwriting}

\begin{figure}[t]
\center
\begin{minipage}[t]{10cm}
\psfig{figure=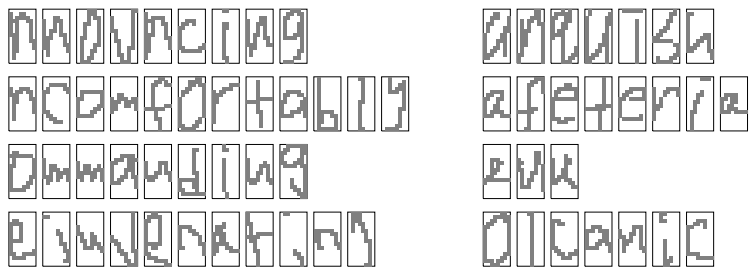,width=10cm}
\end{minipage}
\caption{Eight example words from the handwriting recognition data set.}
\label{fig:handwriting:example}
\end{figure}

The handwriting recognition task we consider was introduced by
\newcite{kassel95handwriting}.  Later, \newcite{taskar03mmmn}
presented state-of-the-art results on this task using max-margin
Markov networks.  The task is an image recognition task: the input is
a sequence of pre-segmented hand-drawn letters and the output is the
character sequence (``a''-``z'') in these images.  The data set we
consider is identical to that considered by \newcite{taskar03mmmn} and
includes 6600 sequences (words) collected from 150 subjects.  The
average word contains $8$ characters.  The images are $8 \times 16$
pixels in size, and rasterized into a binary representation.  Example
image sequences are shown in Figure~\ref{fig:handwriting:example} (the
first characters are removed because they are capitalized).

For each possible output letter, there is a unique feature that counts
how many times that letter appears in the output.  Furthermore, for
each pair of letters, there is an ``edge'' feature counting how many
times this pair appears adjacent in the output.  These edge features
are the \emph{only} ``structural features'' used for this task (i.e.,
features that span multiple output labels).  Finally, for every output
letter and for every pixel position, there is a feature that counts
how many times that pixel position is ``on'' for the given output
letter.

In the experiments, we consider two variants of the data set.  The
first, ``small,'' is the problem considered by \newcite{taskar03mmmn}.
In the small problem, ten fold cross-validation is performed over the
data set; in each fold, roughly 600 words are used as training data
and the remaining 6000 are used as test data.  In addition to this
setting, we also consider the ``large'' reverse experiment: in each
fold, 6000 words are used as training data and 600 are used as test
data.

\subsubsection{Spanish Named Entity Recognition} \label{sec:sequence:problems:sner}

\begin{figure}[t]
El presidente de la [Junta de Extremadura]\lab{ORG} , [Juan Carlos Rodr\'iguez Ibarra]\lab{PER} , recibir\'a en la sede de la [Presidencia del Gobierno]\lab{ORG} extreme\~no a familiares de varios de los condenados por el proceso `` [Lasa-Zabala]\lab{MISC} '' , entre ellos a [Lourdes D\'iez Urraca]\lab{PER} , esposa del ex gobernador civil de [Guip\'uzcoa]\lab{LOC} [Julen Elgorriaga]\lab{PER} ; y a [Antonio Rodr\'iguez Galindo]\lab{PER} , hermano del general [Enrique Rodr\'iguez Galindo]\lab{PER} .
\caption{Example labeled sentence from the Spanish Named Entity Recognition task.}
\label{fig:ner:example}
\end{figure}

The named entity recognition (NER) task is concerned with spotting
names of persons, places and organizations in text.  Moreover, in NER
we only aim to spot \emph{names} and neither pronouns (``he'') nor
nominal references (``the President'').  We use the CoNLL 2002 data
set, which consists of $8324$ training sentences and $1517$ test
sentences; examples are shown in Figure~\ref{fig:ner:example}.  A
$300$-sentence subset of the training data set was previously used by
\newcite{tsochantaridis05svmiso} for evaluating the \SVMISO\ framework
in the context of sequence labeling.  The small training set was
likely used for computational considerations.  The best reported
results to date using the full data set are due to
\newcite{ando05transfer}.  We report results on both the ``small'' and
``large'' data sets.

The structural features used for this task are roughly the same as in
the handwriting recognition case.  For each label, each label pair and
each label triple, a feature counts the number of times this element
is observed in the output.  Furthermore, the standard set of input
features includes the words and simple functions of the words (case
markings, prefix and suffix up to three characters) within a window of
$\pm 2$ around the current position.  These input features are paired
with the current label.  This feature set is fairly standard in the
literature, though \newcite{ando05transfer} report significantly
improved results using a much larger set of features.  In the results
shown later in this section, all comparison algorithms use identical
feature sets.

\subsubsection{Syntactic Chunking} \label{sec:sequence:problems:chunking}

\begin{figure}[t]
[Great American]\lab{NP}
[said]\lab{VP}
[it]\lab{NP}
[increased]\lab{VP}
[its loan-loss reserves]\lab{NP}
[by]\lab{PP}
[\$ 93 million]\lab{NP}
[after]\lab{PP}
[reviewing]\lab{VP}
[its loan portfolio]\lab{NP}
,
[raising]\lab{VP}
[its total loan and real estate reserves]\lab{NP}
[to]\lab{PP}
[\$ 217 million]\lab{NP}
.
\caption{Example labeled sentence from the syntactic chunking task.}
\label{fig:chunking:example}
\end{figure}

The final sequence labeling task we consider is syntactic chunking
(for English), based on the CoNLL 2000 data set.  This data set
includes $8936$ sentences of training data and $2012$ sentences of
test data.  An example is shown in Figure~\ref{fig:chunking:example}.
(Several authors have considered the \emph{noun-phrase chunking} task
instead of the full syntactic chunking task.  It is important to
notice the difference, though results on these two tasks are typically
very similar, indicating that the majority of the difficulty is with
noun phrases.)

We use the same set of features across all models, separated into
``base features'' and ``meta features.''  The base features apply to
words individually, while meta features apply to entire chunks.  The
standard base features used are: the chunk length, the word (original,
lower cased, stemmed, and original-stem), the case pattern of the
word, the first and last 1, 2 and 3 characters, and the part of speech
and its first character.  We additionally consider membership features
for lists of names, locations, abbreviations, stop words, etc.  The
meta features we use are, for any base feature $b$, $b$ at position $i$
(for any sub-position of the chunk), $b$ before/after the chunk, the
entire $b$-sequence in the chunk, and any 2- or 3-gram tuple of $b$s
in the chunk.  We use a first order Markov assumption (chunk label only
depends on the most recent previous label) and all features are placed
on labels, not on transitions.  In the results shown later in this
section, some of the algorithms use a slightly different feature set.
In particular, the CRF-based model uses similar, but not identical
features; see \cite{sutton05bagging} for details.

\subsubsection{Joint Chunking and Tagging} \label{sec:sequence:problems:joint}

In the preceding sections, we considered the single sequence labeling
task: to each element in a sequence, a single label is assigned.  In
this section, we consider the \emph{joint} sequence labeling task.  In
this task, each element in a sequence is labeled with \emph{multiple}
tags.  A canonical example of this task is joint POS tagging and
syntactic chunking \cite{sutton04fcrfs}.  An example sentence jointly
labeled for these two outputs is shown in
Figure~\ref{fig:joint:example} (under the BIO encoding).

\begin{figure}[t]
Great\jointlab{B-NP}{NNP}
American\jointlab{I-NP}{NNP}
said\jointlab{B-VP}{VBD}
it\jointlab{B-NP}{PRP}
increased\jointlab{B-VP}{VBD}
its\jointlab{B-NP}{PRP\$}
loan-loss\jointlab{I-NP}{NN}
reserves\jointlab{I-NP}{NNS}
by\jointlab{B-PP}{IN}
\$\jointlab{B-NP}{\$}
93\jointlab{I-NP}{CD}
million\jointlab{I-NP}{CD}
after\jointlab{B-PP}{IN}
reviewing\jointlab{B-VP}{VBG}
its\jointlab{B-NP}{PRP\$}
loan\jointlab{I-NP}{NN}
portfolio\jointlab{I-NP}{NN}
.\jointlab{O}{.}
\caption{Example sentence for the joint POS tagging and syntactic
  chunking task.}
\label{fig:joint:example}
\end{figure}

For \searn, there is little difference between standard sequence
labeling and joint sequence labeling.  We use the same data set as for
the standard syntactic chunking task
(Section~\ref{sec:sequence:problems:chunking}) and essentially the
same features.  In order to model the fact that the two streams of
labels are not independent, we decompose the problem into two parallel
tagging tasks.  First, the first POS label is determined, then the
first chunk label, then the second POS label, then the second chunk
label, etc.  The only difference between the features we use in this
task and the vanilla chunking task has to do the structural features.
The structural features we use include the obvious Markov features on
the individual sequences: counts of singleton, doubleton and tripleton
POS and chunk tags.  We also use ``crossing sequence'' features.  In
particular, we use counts of pairs of POS and chunk tags at the same
time period as well as pairs of POS tags at time $t$ and chunk tags at
$t-1$ and vice versa.

\subsubsection{Search and Initial Policies} \label{sec:sequence:searn}

The choice of ``search'' algorithm in \searn\ essentially boils down
to the choice of output vector representation, since, as defined,
\searn\ always operates in a left-to-right manner over the output
vector.  In this section, we describe vector representations for the
output space and corresponding optimal policies for \searn.

The most natural vector encoding of the sequence labeling problem is
simply as itself.  In this case, the search proceeds in a greedy
left-to-right manner with one word being labeled per step.  This
search order admits some linguistic plausibility for many natural
language problems.  It is also attractive because (assuming unit-time
classification) it scales as $\order(NL)$, where $N$ is the length of
the input and $L$ is the number of labels, independent of the number
of features or the loss function.  However, this vector encoding is
also highly biased, in the sense that it is perhaps not optimal for
some (perhaps unnatural) problems.  Other orders are possible (such as
allowing any arbitrary position to be labeled at any time, effectively
mimicing belief propagation); see \cite{daume06thesis} for more
experimental results under alternative orderings.

For joint segmentation and labeling tasks, such as named entity
identification and syntactic chunking, there are two natural
encodings: word-at-a-time and chunk-at-a-time.  In word-at-a-time, one
essentially follows the ``BIO encoding'' and tags a single word in
each search step.  In chunk-at-a-time, one tags single \emph{chunks}
in each search step, which can consist of multiple words (after fixing
a maximum phrase length).  In our experiments, we focus exclusively on
chunk-at-a-time decoding, as it is more expressive (feature-wise) and
has been seen to perform better in other scenarios
\cite{sarawagi04scrf}).

Under the chunk-at-a-time encoding, an input of length $N$ leads to a
vector of length $N$ over $M\times L + 1$ labels, where $M$ is the
maximum phrase length.  The interpretation of the first $M\times L$
labels, for instance $(m,l)$ means that the next phrase is of length
$m$ and is a phrase of type $l$.  The ``$+1$'' label corresponds to a
``complete'' indicator.  Any vector for which the sum of the ``$m$''
components is not exactly $N$ attains maximum loss.

\subsubsection{Initial Policies}

For the sequence labeling problem under Hamming loss, the optimal
policy is always to label the next word correctly.  In the
left-to-right order, this is straightforward.  For the segmentation
problem, word-at-a-time and chunk-at-a-time behave very similarly with
respect to the loss function and optimal policy.  We discuss
word-at-a-time because its notationally more convenient, but the
difference is negligible.  The optimal policy can be computed by
analyzing a few options in Eq~\eqref{eq:op:f}

\begin{equation} \label{eq:op:f}
\pi(x,y_{1:T},\hat y_{1:t-1}) =
\brack{
\textrm{begin } X \quad\quad\quad & y_t = \textrm{begin } X \\
\textrm{in } X    & y_t = \textrm{in } X \textrm{ and } \hat y_{t-1} \in \{ \textrm{begin } X, \textrm{in } X \} \\
\textrm{out}      & \textrm{otherwise}
}
\end{equation}

It is easy to show that this policy is optimal (assuming noise-free
training data).  There is, however, another equally optimal policy.
For instance, if $y_t$ is ``in $X$'' but $\hat y_{t-1}$ is ``in $Y$''
(for $X \neq Y$), then it is equally optimal to select $\hat y_t$ to
be ``out'' or ``in $Y$''.  In theory, when the optimal policy does not
care about a particular decision, one can randomize over the
selection.  However, in practice, we always default to a particular
choice to reduce noise in the learning process.

For all of the policies described above, it is also straightforward to
compute the optimal approximation for estimating the expected cost of
an action.  In the Hamming loss case, the loss is $0$ if the choice is
correct and $1$ otherwise.  The computation for F$_1$ loss is a bit
more complicated: one needs to compute an optimal intersection size
for the future and add it to the past ``actual'' size.  This is also
straightforward by analyzing the same cases as in Eq~\eqref{eq:op:f}.

\subsubsection{Experimental Results and Discussion} \label{sec:sequence:comparison}

\begin{table}[t]
\center
\begin{tabular}{|l|cc|cc|c|c|}
\hline
{\bf ALGORITHM} &
  \multicolumn{2}{c|}{\bf Handwriting} &
                         \multicolumn{2}{c|}{\bf NER} &
                                                         {\bf Chunk} &
                                                                        {\bf C+T} \\
&
Small & Large &
Small & Large & & \\
\hline
{\bf CLASSIFICATION} &&&&&& \\
~~~~~{\bf Perceptron} & $65.56$ & $70.05$     & $91.11$ & $94.37$     & $83.12$     & $87.88$ \\
~~~~~{\bf Log Reg}    & $68.65$ & $72.10$     & $93.62$ & $96.09$     & $85.40$     & $90.39$ \\
~~~~~{\bf SVM-Lin}    & $75.75$ & $82.42$     & $93.74$ & $97.31$     & $86.09$     & $93.94$ \\
~~~~~{\bf SVM-Quad}   & $82.63$ & $82.52$     & $85.49$ & $85.49$     & \slow       & \slow   \\
\hline
{\bf STRUCTURED} &&&&&& \\
~~~~~{\bf Str. Perc.} & $69.74$ & $74.12$     & $93.18$ & $95.32$     & $92.44$     & $93.12$ \\
~~~~~{\bf CRF}        & \unav   & \unav       & $94.94$ & \slow       & $94.77$     & $96.48$ \\
~~~~~{\bf \SVMISO}     & \unav   & \unav       & $94.90$ & \slow       & \unav       & \unav   \\
~~~~~{\bf \mmmn-Lin}  & $81.00$ & \slow       & \unav   & \unav       & \unav       & \unav   \\
~~~~~{\bf \mmmn-Quad} & $87.00$ & \slow       & \unav   & \unav       & \unav       & \unav   \\
\hline
{\bf SEARN} &&&&&& \\
~~~~~{\bf Perceptron} & $70.17$ & $76.88$     & $95.01$ & $97.67$     & $94.36$     & $96.81$ \\
~~~~~{\bf Log Reg}    & $73.81$ & $79.28$     & $95.90$ & $98.17$     & $94.47$     & $96.95$ \\
~~~~~{\bf SVM-Lin}    & $82.12$ & $90.58$     & $95.91$ & $98.11$     & $94.44$     & $96.98$ \\
~~~~~{\bf SVM-Quad}   & $87.55$ & $90.91$     & $89.31$ & $90.01$     & \slow       & \slow   \\
\hline
\end{tabular}
\caption{Empirical comparison of performance of alternative structured
  prediction algorithms against \searn\ on sequence labeling tasks.
  (Top) Comparison for whole-sequence 0/1 loss; (Bottom) Comparison
  for individual losses: Hamming for handwriting and Chunking+Tagging
  and F for NER and Chunking.  \searn\ is always optimized for the
  appropriate loss.}
\label{tab:sequence:comparison}
\end{table}

In this section, we compare the performance of \searn\ to the
performance of alternative structured prediction techniques over the
data sets described above.  The results of this evaluation are shown
in Table~\ref{tab:sequence:comparison}.  In this table, we compare raw
classification algorithms (perceptron, logistic regression and SVMs)
to alternative structured prediction algorithms (structured
perceptron, CRFs, \SVMISOs\ and \mmmns) to \searn\ with three baseline
classifiers (perceptron, logistic regression and SVMs).  For all SVM
algorithms and for \mmmns, we compare both linear and quadratic
kernels (cubic kernels were evaluated but did not lead to improved
performance over quadratic kernels).

For all \searn-based models, we use the the following settings of the
tunable parameters (see \cite{daume06thesis} for a comparison of
different settings).  We use the optimal approximation for the
computation of the per-action costs.  We use a left-to-right search
order with a beam of size 10.  For the chunking tasks, we use
chunk-at-a-time search.  We use weighted all pairs and costing to
reduce from cost-sensitive classification to binary classification.

Note that some entries in Table~\ref{tab:sequence:comparison} are
missing.  The vast majority of these entries are missing because the
algorithm considered could not reasonably scale to the data set under
consideration.  These are indicated with a ``\slow'' symbol.  Other
entries are not available simply because the results we report are
copied from other publications and these publications did not report
all relevant scores.  These are indicated with a ``\unav'' symbol.

We observe several patterns in the results from
Table~\ref{tab:sequence:comparison}.  The first is that structured
techniques consistently outperform their classification counterparts
(eg., CRFs outperform logistic regression).  The single exception is
on the small handwriting task: the quadratic SVM outperforms the
quadratic \mmmn.\footnote{However, it should be noted that a different
implementation technique was used in this comparison.  The \mmmn\ is
based on an SMO algorithm, while the quadratic SVM is libsvm
\cite{libsvm}.}  For all classifiers, adding \searn\ consistently
improves performance.

An obvious pattern worth noticing is that moving from the small data
set to the large data set results in improved performance, regardless
of learning algorithm.  However, equally interesting is that simple
classification techniques when applied to large data sets outperform
complicated learning techniques applied to small data sets.  Although
this comparison is not completely fair---both algorithms should get
access to the same data---if the algorithm (like the \SVMISO\ or the
\mmmn) cannot scale to the large data set, then something is missing.
For instance, a vanilla SVM on the large handwriting data set
outperforms the \mmmn\ on the small set.  Similarly, a vanilla
logistic regression classifier trained on the large NER data set
outperforms the \SVMISO\ and the CRF on the small data sets.

On the same data set, \searn\ can perform comparably or better than
competing structured prediction techniques.  On the small handwriting
task, the two best performing systems are \mmmns\ with quadratic
kernels ($87.0\%$ accuracy) and \searn\ with quadratic SVMs ($87.6\%$
accuracy).  On the NER task, \searn\ with a perceptron classifier
performs comparably to \SVMISO\ and CRFs (at around $95.9\%$
accuracy).  On the Chunking+Tagging task, all varieties of \searn\
perform comparatively to the CRF.  In fact, the only task on which
\searn\ does \emph{not} outperform the competing techniques is on the
raw chunking task, for which the CRF obtains an F-score of $94.77$
compared to $94.47$ for \searn, using a significantly different
feature set.

The final result from Table~\ref{tab:sequence:comparison} worth
noticing is that, with the exception of the handwriting recognition
task, \searn\ using logistic regression as a base learner performs at
the top of the pack.  The SVM-based \searn\ models typically perform
slightly better, but not significantly.  In fact, the raw averaged
perceptron with \searn\ performs almost as well as the logistic
regression.  This is a nice result because the SVM-based models tend
to be expensive to train, especially in comparison to the perceptron.
The fact that this pattern does not hold for the handwriting task is
likely due to the fact that the data for this task is quite unlike the
data for the other tasks.  For the handwriting task, there are a
comparatively small number of features which are individually much less
predictive of the class.  It is only in combination that good
classifiers can be learned.

While these results are useful, they should be taken with a grain of
salt.  Sequence labeling is a very easy problem.  The structure is
simple and the most common loss functions decompose over the
structure.  The comparatively good performance of raw classifiers
suggests that the importance of structure is minor.  In fact, some
results suggest that one need not actually consider the structure at
all for some such problems
\cite{punyakanok01inference,punyakanok05constrainted}.

\subsection{Automatic Document Summarization} \label{sec:summarization}

Multidocument summarization is the task of creating a summary out of a
collection of documents on a focused topic.  In query-focused
summarization, this topic is given explicitly in the form of a user's
query.  The dominant approach to the multidocument summarization
problem is sentence extraction: a summary is created by greedily
extracting sentences from the document collection until a pre-defined
word limit is reached.  \newcite{teufel97sentence} and
\newcite{lin02multi} describe representative examples.  Recent work in
sentence compression \cite{knight-marcu02,mcdonald06compression} and
document compression \cite{daume02noisychannel} attempts to take small
steps beyond sentence extraction.  Compression models can be seen as
techniques for extracting sentences then dropping extraneous
information.  They are more powerful than simple sentence extraction
systems, while remaining trainable and tractable.  Unfortunately,
their training hinges on the existence of $\langle$ sentence,
compression $\rangle$ pairs, where the compression is obtained from
the sentence by only dropping words and phrases (the work of
\newcite{turner05compression} is an exception).  Obtaining such data
is quite challenging.

The exact model we use for the document summarization task is a novel
``vine-growth'' model, described in more detail in
\cite{daume06thesis}.  The vine-growth method uses syntactic parses of
the sentence in the form of \emph{dependency} structures.  In the
vine-growth model, if a word $w$ is to be included in the summary,
then all words closer to the tree root are included.

\subsubsection{Search Space and Actions}

The search algorithm we employ for implementing the vine-growth model
is based on incrementally \emph{growing} summaries.  In essence,
beginning with an empty summary, the algorithm incrementally adds
words to the summary, either by beginning a new sentence or growing
existing sentences.  At any step in search, the root of a new sentence
may be added, as may any direct child of a previously added node.
To see more clearly how the vine-growth model functions, consider
Figure~\ref{fig:example-sentence}.  This figure shows a four step
process for creating the summary ``the man ate a sandwich .'' from the
original document sentence ``the man ate a big sandwich with pickles
.''

\begin{figure}[t]
\psfig{figure=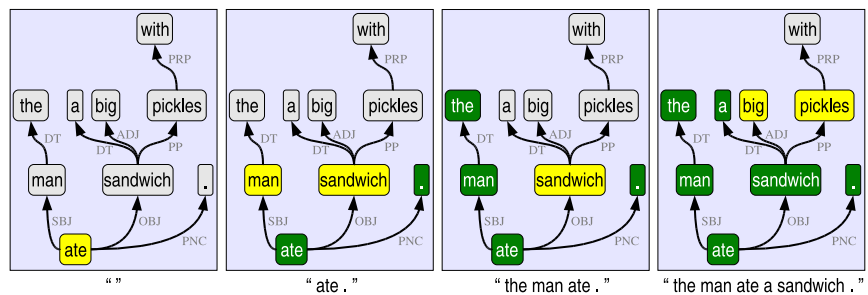,width=11.5cm}
\caption{An example of the creation of a summary under the vine-growth
  model.}
\label{fig:example-sentence}
\end{figure}

When there is more than one sentence in the source documents, the
search proceeds asynchronously across all sentences.  When the
sentences are laid out adjacently, the end summary is obtained by
taking all the green summary nodes once a pre-defined word limit has
been reached.  This final summary is a collection of subtrees grown
off a sequence of underlying trees: hence the name ``vine-growth.''

\subsubsection{Data and Evaluation Criteria} \label{sec:summarization:eval}

For data, we use the DUC 2005 data set \cite{duc05}.  This consists of
50 document collections of 25 documents each; each document collection
includes a human-written query.  Each document collection additionally
has five human-written ``reference'' summaries (250 words long, each)
that serve as the gold standard.  In the official DUC evaluations, all
50 collections are ``test data.''  However, since the DUC 2005 task is
significantly different from previous DUC tasks, there is no a good
source of training data.  Therefore, we report results based on
10-fold cross validation.  We train on 45 collections and test on the
remaining 5.

Automatic evaluation is a notoriously difficult problem for document
summarization.  The current popular choice for metric is Rouge
\cite{lin-naacl03}, which (roughly speaking) computes $n$-gram overlap
between a system summary and a set of human written summaries.  In
various experiments, Rouge has been seen to correspond with human
judgment of summary quality.  In the experiments described in this
chapter, we use the ``Rouge 2'' metric, which uses evenly weighted
bigram scores.

\subsubsection{Initial Policy} \label{sec:summarization:searn}

Computing the best label completion under Rouge metric for the
vine-growth model is intractable.  The intractability stems from the
model constraint that a word can only be added to a summary after its
parent is added.  We therefore use an approximate, search-based policy
(see Section~\ref{sec:searn:policy:search-based}).  In order to
approximate the cost of a given partial summary, we \emph{search} for
the best possible completion.  That is, if our goal is a 100 word summary
and we have already created a 50 word summary, then we execute beam
search (beam size $20$) for the remaining 50 words that maximize the
Rouge score.

\subsubsection{Feature Functions}

Features in the vine-growth model may consider any aspect of the
currently generated summary, and any part of the input document set.
These features include simple lexical features: word identity, stem
and part of speech of the word under consideration, the syntactic
relation with its parent, the position and length of the sentence it
appears in, whether it appears in quotes, the length of the document
it appears in, the number of pronouns and attribution verbs in the
subtree rooted at the word.  The features also include language model
probabilities for: the word, sentence and subtree under language
models derived from the query, a \bayesum\ representation of the
query, and the existing partial summary.

\subsubsection{Experimental Results}

Experimental results are shown in
Table~\ref{tab:results:summarization}.  We report Rouge scores for
summaries of length 100 and length 250.  We compare the following
systems.  First, oracle systems that perform the summarization task
\emph{with} knowledge of the true output, attempting to maximize the
Rouge score.  We present results for an oracle sentence extraction
system (Extr) and an oracle vine-growth system (Vine).  Second, we
present the results of the \searn-based systems, again for both
sentence extraction (Extr) and vine-growth (Vine).  Both of these are
trained with respect to the oracle system.  (Note that it is
impossible to compare against competing structured prediction
techniques.  This summarization problem, even in its simplified form,
is far too complex to be amenable to other methods.)  For comparison,
we present results from the \bayesum\ system
\cite{daume05duc,daume06bqfs}, which achieved the highest score
according to human evaluations of responsiveness in DUC 05.  This
system, as submitted to DUC 05, was \emph{trained} on DUC 2003 data;
the results for this configuration are shown in the ``D03'' column.
For the sake of fair comparison, we also present the results of this
system, trained in the same cross-validation approach as the
\searn-based systems (column ``D05'').  Finally, we present the
results for the baseline system and for the best DUC 2005 system
(according to the Rouge 2 metric).

\begin{table}[t]
\center
\begin{small}
\begin{tabular}{|l||c|c||c|c||c|c||c|c|}
\hline
 & \multicolumn{2}{c||}{\bf ORACLE}
 & \multicolumn{2}{c||}{\bf SEARN}
 & \multicolumn{2}{c||}{\bf BAYESUM}
 & & \\
 & Vine & Extr
 & Vine & Extr
 & D05 & D03
 & Base & Best \\
\hline
{\bf 100 w}
 & .0729
 & .0362
 & .0415
 & .0345
 & .0340
 & .0316
 & .0181
 & - \\
{\bf 250 w}
 & .1351
 & .0809
 & .0824
 & .0767
 & .0762
 & .0698
 & .0403
 & .0725 \\
\hline
\end{tabular}
\end{small}
\caption{Summarization results; values are Rouge 2 scores (higher is better).}
\label{tab:results:summarization}
\end{table}

As we can see from Table~\ref{tab:results:summarization} at the 100
word level, sentence extraction is a nearly solved problem for this
domain and this evaluation metric.  That is, the oracle sentence
extraction system yields a Rouge score of $0.0362$, compared to the
score achieved by the \searn\ system of $0.0345$.  This difference is
on the border of statistical significance at the $95\%$ level.  The
next noticeable item in the results is that, although the \searn-based
\emph{extraction} system comes quite close to the theoretical optimal,
the oracle results for the \emph{vine-growth} method are significantly
higher.  Not surprisingly, under \searn, the summaries produced by the
vine-growth technique are uniformally better than those produced by
raw extraction.  The last aspect of the results to notice is how the
\searn-based models compare to the best DUC 2005 system, which
achieved a Rouge score of $0.0725$.  The \searn-based systems
uniformly dominate this result, but this comparison is not fair due to
the training data.  We can approximate the expected improvement for
having the new training data by comparing the \bayesum\ system when
trained on the DUC 2005 and DUC 2003 data: the improvement is $0.0064$
absolute.  When this result is added to the best DUC 2005 system, its
score rises to $0.0789$, which is better than the \searn-based
extraction system but not as good as the vine-growth system.  It
should be noted that the best DUC 2005 system \emph{was} a purely
extractive system \cite{ye05duc}.

\section{Discussion and Conclusions} \label{sec:discussion}

In this paper, we have:

\begin{itemize}
\item Presented an algorithm, \searn, for solving complex structured
  prediction problems with minimal assumptions on the structure of the
  output and loss function.

\item Compared the performance of \searn\ against standard structured
  prediction algorithms on standard sequence labeling tasks, showing
  that it is competitive with existing techniques.

\item Described a novel approach to summarization---the vine-growth
  method---and applied \searn\ to the underlying learning problem,
  yielding state-of-the-art performance on standardized summarization
  data sets.
\end{itemize}

There are many lenses through which one can view the \searn\
algorithm.

From an applied perspective, \searn\ is an easy technique for training
models for which complex search algorithms must be used.  For
instance, when using multiclass logistic regression as a base
classifier for Hamming loss, the first iteration of \searn\ is
identical to training a maximum entropy Markov model.  The subsequent
iterations of \searn\ can be seen as attempting to get around the fact
that MEMMs are trained \emph{assuming} all previous decisions are made
correctly.  This assumption is false, of course, in practice.  Similar
recent algorithms such a decision-tree-based parsing
\cite{turian06parsing} and perceptron-based machine translation
\cite{liang06endtoend} can also be seen as running a (slightly
modified) first iteration of \searn.

\searn\ contrasts with more typical algorithms such as CRFs and
\mmmns\ based on considering how information is shared at test time.
Standard algorithms use exact (typically Viterbi) search to share full
information across the entire output, ``trading off'' one decision for
another.  \searn\ takes an alternative approach: it attempts to share
information at \emph{training time.}  In particular, by training the
classifier using a loss based on both past experience and future
expectations, the training attempts to integrate this information
during learning.  This is not unsimilar to the ``alternative
objective'' proposed by \newcite{kakade02objective} for CRFs.  One
approach is not necessarily better than the other; they are simply
different ways to accomplish the same goal.

One potential limitation to \searn\ is that when one trains a new
classifier on the output of a previous iteration's classifier, it is
usually going to be the case that previous iteration's classifier
performs better on the training data than on the test data.  This
means that, although training via \searn\ is likely preferable to
training against \emph{only} an initial policy, it can still be overly
optimistic.  Based on the experimental evidence, it appears that this
has yet to be a serious concern, but it remains worrisome.  There are
two easy ways to combat this problem.  The first is simply to attempt
to ensure that the learned classifiers do not overfit at all.  In
practice, however, this can be difficult.  Another approach with a
high computational cost is cross-validation.  Instead of training one
classifier in each \searn\ step, one could train ten, each holding out
a different $10\%$ of the data.  When asked to run the ``current''
classifier on an example, the classifier not trained on the example is
used.  This does not completely remove the possiblity of overfitting,
but significantly lessens its likelihood.

A second limitation, pointed out by \newcite{zhang06searn}, is that
there is a slight disparity between what \searn\ does at a theoretical
level and how \searn\ functions in practice.  In particular, \searn\
does not \emph{actually} start with the optimal policy.  Even when we
can compute the initial policy exactly, the ``true outputs'' on which
this initial policy are based are potentially noisy.  This means that
while $\pi$ is optimal for the \emph{noisy} data, it is not optimal
for the \emph{true} data distribution.  In fact, it is possible to
construct noisy distributions where \searn\ performs poorly.\footnote{
One can construct such a noisy distribution as follows.  Suppose there
is fundamental noise and a ``safe'' option which results in small
loss.  Suppose this safe option is always more than a one step
deviation from the highly noisy ``optimal'' sequence.  \searn\ can be
confused by this divergence.}  Finding other initial policies which
are closer to optimal in these situations is an open problem.

\searn\ obeys a desirable theoretical property: given a good
classification algorithm, one is guaranteed a good structured
prediction algorithm.  Importantly, this result is independent of the
size of the search space or the tractability of the search method.
This shows that local learning---when done properly---can lead to good
global performance.  From the perspective of applied machine learning,
\searn\ serves as an interpreter through which engineers can easily
make use of state-of-the-art machine learning techniques.

In the context of structured prediction algorithms, \searn\ lies
somewhere between global learning algorithms, such as \mmmns\ and
CRFs, and local learning algorithms, such as those described
\newcite{punyakanok01inference}.  The key difference between \searn\
and global algorithms is in how uncertainty is handled.  In global
algorithms, the search algorithm is used at test time to propagate
uncertainty across the structure.  In \searn, the prediction costs are
used during training time to propagate uncertainty across the
structure.  Both contrast with local learning, in which no uncertainty
is propagated.

From a wider machine learning perspective, \searn\ makes more apparent
the connection between reinforcement learning and structured
prediction.  In particular, structured prediction can be viewed as a
reinforcement learning problem in a degenerate world in which all
observations are available at the initial time step.  However, there
are clearly alternative middle-grounds between pure structured
prediction and full-blown reinforcement learning (and natural
applications---such as planning---in this realm) for which this
connection might serve to be useful.

Despite these successes, there is much future work that is possible.
One significant open question on the theoretical side is that of
sample complexity: ``How many examples do we need in order to achieve
learning under additional assumptions?''  Related problems of
semi-supervised and active learning in the \searn\ framework are also
interesting and likely to produce powerful extensions.  Another vein
of research is in applying \searn to domains other than language.
Structured prediction problems arise in a large variety of settings
(vision, biology, system design, compilers, etc.).  For each of these
domains, different sorts of search algorithms and different sorts of
features are necessary.  Although \searn\ has been discussed largely
as a method for solving structured prediction problems, it is, more
generally, a method for integrating \emph{search} and \emph{learning}.
This leads to potential applications of \searn\ that fall strictly
outside the scope of structured prediction.

\bibliographystyle{plain}
\bibliography{bibfile}

\end{document}